\documentclass{article}

\PassOptionsToPackage{numbers, compress}{natbib}
     \usepackage[preprint]{neurips_2019}

\usepackage[utf8]{inputenc} % allow utf-8 input
\usepackage[T1]{fontenc}    % use 8-bit T1 fonts
\usepackage{hyperref}       % hyperlinks
\usepackage{url}            % simple URL typesetting
\usepackage{booktabs}       % professional-quality tables
\usepackage{amsfonts}       % blackboard math symbols
\usepackage{nicefrac}       % compact symbols for 1/2, etc.
\usepackage{microtype}      % microtypography
\usepackage{amsmath,,amsthm}
\usepackage{bm}
\usepackage{graphicx}
\usepackage{wrapfig}
\usepackage{subcaption}

\usepackage{algorithm}
\usepackage{algorithmic}

\newtheorem{lemma}{Lemma}

\title{Causal Discovery by Kernel Intrinsic \\ Invariance Measure}

\author{%
  Zhitang Chen\quad Shengyu Zhu\quad Yue Liu\quad Tim Tse\\
  Huawei Noah's Ark Lab \\
  \texttt{\{chenzhitang2,zhushengyu,liuyue52,tim.tse\}@huawei.com}
}

\begin{document}

\maketitle

\begin{abstract}
Reasoning based on causality, instead of association has been considered as a key ingredient towards real machine intelligence. However, it is a challenging task to infer causal relationship/structure among variables. In recent years, an Independent Mechanism (IM) principle was proposed, stating that the mechanism generating the cause and the one mapping the cause to the effect are independent. As the conjecture, it is argued that in the causal direction, the conditional distributions instantiated at different value of the conditioning variable have less variation than the anti-causal direction. Existing state-of-the-arts simply compare the variance of the RKHS mean embedding norms of these conditional distributions. In this paper, we prove that this norm-based approach sacrifices important information of the original conditional distributions. We propose a Kernel Intrinsic Invariance Measure (KIIM) to capture higher order statistics corresponding to the shapes of the density functions. We show our algorithm can be reduced to an eigen-decomposition task on a kernel matrix measuring intrinsic deviance/invariance. Causal directions can then be inferred by comparing the KIIM scores of two hypothetic directions. Experiments on synthetic and real data are conducted to show the advantages of our methods over existing solutions.
\end{abstract}

\section{Introduction}

Recent breakthrough in deep learning has been significantly advancing Artificial Intelligence (AI). We witness great success of deep learning in many applications such as image classification, image recognition, speech recognition, natural language processing etc. Deep learning methods, or specifically deep neural networks have become the dominant approach for machine learning and AI and thus attracts tremendous amount of attention from both the academia and the industry. However, there are still a number of open challenges remained to tackle for deep learning. To name a few, heavy demand on labeled data, bad generalizability, vulnerable against adversarial attacks and lack of interpretability of deep learning methods are the most notorious ones. Recently, it is advocated in the AI community that causality might be one of the tools to solve the aforementioned open problems. It has been argued that causality, instead of ``superficial association" is invariant cross domain. Machine learning algorithms that learn, and utilize the causal relationship amongst variables provide better generalization performance, robustness against adversarial attacks and better interpretability. Besides the area of machine learning and AI \cite{bengio2019meta,scholkopf2013semi,peters2016causal,lopez2017discovering}, causal discovery also play an important role in economics, sociology, bioinformatic and medical science etc.

However, how to unveil the causal relationship among variables from pure observational (or post-intervention) data is challenging. A bunch of methods have been proposed in the past three decades including Bayesian network \cite{pearl2003causality}, Structural Equation Models (SEM) \cite{shimizu2006linear,hoyer2009nonlinear,heinze2018causal}. However, these methods have their limitations. For example, Bayesian networks via constained-based approach or score-based approach are not able to fully identify the ground-truth graphs but only up to ``Markov equivalent class" \cite{pearl2003causality}. In addition, they are not able to solve the more fundamental problem, i.e., causal discovery for a cause-effect pair.

\begin{wrapfigure}{r}{0.5\textwidth}
%\begin{figure}[h]
\vspace{-6mm}
  \centering
  \includegraphics[width=0.5\textwidth]{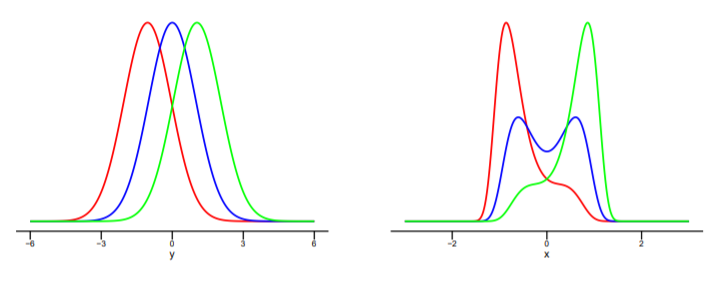}
  \caption{An example of weak discriminative power of RKHS norm \cite{mitrovic2018causal}.}\label{Fig:structural_variation}
%\end{figure}
\end{wrapfigure}

To solve these problems, researchers have been working out new theory and algorithms which try to dig out more regularities from the data distribution \cite{spirtes2016causal}. Amongst them, the Independent Mechanism (IM) principle \cite{janzing2010causal} is considered to be a promising direction. The basic idea behind the IM principle is that nature is parsimonious in the sense that the mechanism generating the cause and the mechanism mapping the cause to the effects are independent, i.e. the probability distribution of the cause $P(X)$ and the conditional distribution mapping the cause to the effect $P(Y|X)$ contain no information of each other. It has been shown that the factorization of the joint distribution according to the causal direction usually yield simpler terms than that in the anti-causal direction \cite{sun2008causal,janzing2010causal}, i.e.
\begin{equation}
\mathcal{K}(P(X)) + \mathcal{K}(P(Y|X))
\leq \mathcal{K}(P(Y)) + \mathcal{K}(P(X|Y)),
\end{equation}
where $\mathcal{K}(\cdot)$ denotes the Kolmogorov complexity which is essentially not computable. Researchers have been proposing computable metrics including RKHS norm \cite{sun2008causal,chen2014causal}, Minimal Description Length \cite{budhathoki2017mdl} etc., to mimic the Kolmogorov complexity in order to derive a practical algorithm for pairwise causal discovery. Our method proposed in this paper falls into this category. According to the IM principle, the conditional distribution $P(Y|X)$ does not depend on $P(X)$ which naturally leads to the conjecture that intrinsic information, e.g. higher order central moments that characterize the shape of $P(Y|X=x)$ does not essentially depend on the value of $x$. In this paper, we prove that existing state-of-the-art norm-based approach along this direction is not sufficient as it sacrifices important information of the original conditional distributions. Instead, we propose a Kernel Intrinsic Invariance Measure (KIIM) to capture the intrinsic invariance of the conditional distribution, i.e. the higher order statistics corresponding to the shapes of the density functions. We show our algorithm can be reduced to an eigen-decomposition task on a kernel matrix measuring intrinsic deviance/invariance.

The rest of the paper is organized as follows: in Sec.\ref{Sec:related_work}, we introduce the basic idea of a recent state-of-the-art method named Kernel Deviance Measure and its limitation; in Sec.\ref{Sec:preliminary}, we give a brief introduction to Reproducing Kernel Hilbert Space (RKHS) embeddings which serve as the tool of our method; in Sec.\ref{Sec:KIIM}, we give a rigorous justification of the limitation of existing methods and then show how our proposed method address those issues; in Sec.\ref{Sec:experiment}, we verify the effectiveness of our proposed method followed by a conclusion in Sec.\ref{Sec:conclusion}.

%However, a recent interview of Judea Pearl \footnote{https://www.theatlantic.com/technology/archive/2018/05/machine-learning-is-stuck-on-asking-why/560675/}, who was the 2011 winner of the ACM Turing Award, the highest distinction in computer science for his ``fundamental contributions to artificial intelligence through the development of a calculus for probabilistic and causal reasoning", has caught the attention of the AI and Machine Learning community on the limitation of deep learning and the necessity of more focus from the whole community on causal reasoning. Pearl, as the pioneer of Machine Learning, became one of its sharpest critics, because as he sees it, \textit{the field of AI got mired in probabilistic association. The state of the art in artificial intelligence today is merely a souped-up version of what machines could already do a generation ago: find hidden regularities in a large set of data.} He said \textit{All the impressive achievements of deep learning amount to just curve fitting}. In order to realize human-level intelligence, Pearl expect that causal reasoning could be the machines.

\section{Related Work}\label{Sec:related_work}

Recently, authors in \cite{mitrovic2018causal} proposed an idea which exploits the variation of the conditional distribution of the hypothetic effect given the hypothetic cause. They argued that the there is less variability in the causal direction than that in the anti-causal direction. An motivating example that is used in \cite{mitrovic2018causal} as follows. Suppose we have two random variables that follow the generating mechanism as $y=x^3+x+\epsilon$, where $\epsilon\sim \mathcal{N}(0,1)$.
As illustrated in Fig. \ref{Fig:structural_variation}, it is obvious that the conditional distribution $p(Y|x)$ instantiated at different value  $x$ are almost identical except for the location; however in the anti-causal direction, the conditional distribution $p(X|y)$ instantiated at different value $y$  have significant structural variation including the number of modes, skewness, kurtosis ect. This piece of structural variation in conditional distributions leads to the so-called   ``cause-effect asymmetry" for causal discovery. The basic idea is to investigate how invariant the conditional distribution (instantiated at different values) is and one prefers the direction with less variation or in other words, more invariance. To achieve it, they proposed the following Kernel Deviance Measure:

\begin{equation}
\mathcal{S}_{\mathbf{x}\rightarrow \mathbf{y}} := {1}/{n} \sum_i ( \Vert \mu_{Y|\mathbf{x}_i} \Vert_{\mathcal{H}_{\mathcal{Y}}} - {1}/{n} \sum_j \Vert  \mu_{Y|\mathbf{x}_j}\Vert_{\mathcal{H}_{\mathcal{Y}}} )^2,
\end{equation}
where $\mathcal{H}$ is a Reproducing Kernel Hilbert Space (RKHS) entailed by a positive definite kernel $k(\cdot, \cdot)$ and $\mu_{Y|\mathbf{x}_i}$ is the kernel mean embedding of the conditional distribution $p(\mathbf{y}|\mathbf{x})$ instantiated at $\mathbf{x}_i$.

The causal discovery rule is straightforward by comparing the scores, i.e. $\mathbf{x} \rightarrow \mathbf{y}$, if $\mathcal{S}_{\mathbf{x}\rightarrow \mathbf{y}} < \mathcal{S}_{\mathbf{y}\rightarrow \mathbf{x}}$, $\mathbf{y} \rightarrow \mathbf{x}$, if $\mathcal{S}_{\mathbf{x}\rightarrow \mathbf{y}} > \mathcal{S}_{\mathbf{y}\rightarrow \mathbf{x}}$, otherwise no conclusion is drawn.
%\begin{equation}
%\left\{
%\begin{array}{lr}
%\mathbf{x} \rightarrow \mathbf{y}, &~\mbox{if} ~ \mathcal{S}_{\mathbf{x}\rightarrow \mathbf{y}} < \mathcal{S}_{\mathbf{y}\rightarrow \mathbf{x}}\\
%\mathbf{y} \rightarrow \mathbf{x}, &~\mbox{if} ~ \mathcal{S}_{\mathbf{x}\rightarrow\mathbf{y}} > \mathcal{S}_{\mathbf{y}\rightarrow \mathbf{x}}\\
%\mbox{no conclusion},&\mbox{otherwise}
%\end{array}\right.
%\end{equation}.

Positive results were reported in \cite{mitrovic2018causal} which suggests that causal discovery via invariance is a promising direction. However, we notice that the above method has significant limitation that should be addressed. Before we introducing our method, we give some preliminary knowledge on RKHS embeddings.
\section{Preliminary on Reproducing Kernel Hilbert Space Embeddings}\label{Sec:preliminary}

Kernel methods \cite{scholkopf2001learning} are a class of machine learning algorithms that map the data from the original space implicitly to a high dimensional or even infinite dimensional feature space $\mathcal{H}$. One can get rid of computing the coordinates of the data in that space explicitly if the algorithm can reduce to inner products of  feature vectors of all data points which can be easily calculated as the kernel function of any two data points. This is called the kernel trick \cite{scholkopf2001learning}. The kernel function essentially act as a similarity function between a pair of data points and thus kernel methods are categorized as a typical method of instant-based learning.
\[
\mathbf{x} \mapsto \phi(\mathbf{x}):=k(\cdot,\mathbf{x}),
\]
where $\langle \phi(\mathbf{x}),\phi(\mathbf{x}')\rangle = \langle k(\cdot,\mathbf{x}), k(\cdot, \mathbf{x}')\rangle= k(\mathbf{x},\mathbf{x}')$ and $k(\mathbf{x},\mathbf{x}')$ is a positive definite kernel.
The kernel mean embedding \cite{smola2007hilbert} of a probability density $p(\mathbf{x})$ is defined as:
\begin{equation}
\mu_X = \int \phi(\mathbf{x}) p(\mathbf{x})d\mathbf{x},
\end{equation}
One can simply interpret the kernel mean embedding as the vector of (higher order) moments. This interpretation is exactly true if one uses a polynomial kernel $k(\mathbf{x},\mathbf{x}')=(\mathbf{x}^T\mathbf{x}'+1)^d$, where $d>0$. It has been shown that if the kernel is characteristic \cite{song2013kernel}, e.g. a Gaussian kernel, then the mapping a probability distribution to its kernel mean embedding is injective, i.e. we lose no information during the mapping. The conditional embeddings of the conditional distribution $p(Y|X)$ is a sweep out a family of points in the RKHS \cite{song2013kernel}, each one of which is essentially the kernel mean embedding of the conditional distribution $p(Y|\mathbf{x})$ indexed by a fixed value of the conditioning variable $\mathbf{x}$. It is shown in \cite{song2013kernel} that under a mild assumption that $\mathbb{E}_{Y|\cdot}[g(Y)]\in \mathcal{H}_{\mathcal{X}}$, the conditional mean embedding can be obtained by Eq.\ref{Eq:conditional_embedding}:
\begin{equation}\label{Eq:conditional_embedding}
\mu_{Y|\mathbf{x}} = \mathcal{C}_{YX}\mathcal{C}_{XX}^{-1}\phi(\mathbf{x}),
\end{equation}
where $\mathcal{C}_{YX}:=\int \phi(\mathbf{y})\otimes \phi(\mathbf{x}) p(\mathbf{x},\mathbf{y})d\mathbf{x}d\mathbf{y}$ and $\mathcal{C}_{XX}:=\int \phi(\mathbf{x})\otimes \phi(\mathbf{x}) p(\mathbf{x})d\mathbf{x}$.
The empirical estimation of the kernel mean embedding and the conditional mean embedding given a set of observation $X=[\mathbf{x}_1,\mathbf{x}_2,\cdots,\mathbf{x}_n]$ and $Y=[\mathbf{y}_1,\mathbf{y}_2,\cdots,\mathbf{y}_n]$:
\begin{equation}\label{Eq:empirical_estimation}
\begin{split}
&\hat{\mu}_X = \frac{1}{n}\bm{\Phi}\mathbf{1},\\
&\hat{\mu}_{Y|\mathbf{x}} = \bm{\Psi}(\mathbf{H}\mathbf{K}_x + \lambda n \mathbf{I})^{-1}\mathbf{k}_{\mathbf{x}},
\end{split}
\end{equation}
where $\bm{\Phi}=[\phi(\mathbf{x}_1),\phi(\mathbf{x}_2),\cdots,\phi(\mathbf{x}_n)]$, $\bm{\Psi}=[\phi(\mathbf{y}_1),\phi(\mathbf{y}_2),\cdots,\phi(\mathbf{y}_n)]$, $\mathbf{K}_x$ is the kernel Gram matrix of $\mathbf{x}$, i.e. $[\mathbf{K}_x]_{ij}=k(\mathbf{x}_i, \mathbf{x}_j)$ and $\mathbf{H}=\mathbf{I}-\frac{1}{n}\mathbf{1}\mathbf{1}^T$ with $\mathbf{I}$ as an identity matrix and $\mathbf{1}$ is a vector of $1s$ of appropriate dimension.
The (conditional) kernel mean embeddings provide compact and nonparametric representation of the (conditional) distribution. Manipulations of the probability distribution such as complicated operations on probability distribution in Bayesian inference can easily reduces to matrix manipulation in the RKHS. For example, Maximum Mean Discrepancy (MMD)\cite{gretton2008kernel} was proposed for two sample test. A Kernel Bayes Rule (KBR) \cite{fukumizu2013kernel} was proposed to conduct Bayesian inference in the RKHS. Given that the RKHS embedding has solid theoretical support and is easy to use, it is adopted in our paper to measure the intrinsic invariance of the conditional distributions.

%\begin{definition}\label{Def:conditionalembedding}
%The operator $\mathcal{U}_{Y|X}$ is defined as: $\mathcal{U}_{Y|X}:= \mathcal{C}_{YX}\mathcal{C}_{XX}^{-1}$.
%\end{definition}
%
%We can now show that under our definition, $\mathcal{U}_{Y|X}$ satisfies the properites that we wanted.
%\begin{theorem}\label{Theorem:satisfyproperties}
%Assuming that $\mathbf{E}_{Y|X}[g(Y)|X]\in\mathcal{H}$, the embedding of conditional distributions in Definition \ref{Def:conditionalembedding} satisfies the properties conditions.
%\end{theorem}

\section{Causal Discovery by Intrinsic Invariance}\label{Sec:KIIM}
%In this section, we investigate causal discovery problems for a cause-effect pair $\mathbf{x}$ and $\mathbf{y}$.
%
%Suppose the ground truth causal direction is $\mathbf{x}\rightarrow \mathbf{y}$ and we aim at identifying the causal direction based on
%pure observational data $\mathbf{X}=[\mathbf{x}_1,\mathbf{x}_2, \cdots, \mathbf{x}_n]$ and $\mathbf{Y}=[\mathbf{y}_1, \mathbf{y}_2, \cdots, \mathbf{y}_n]$.
%
%Causal discovery from pure observational data is very challenging.  Researchers have proposed a number of methods to tackle this task. In recent years, the Independent Mechanism principle (IM) has attracted increasing amount of attention from the machine learning and statistics community.
%
%The basic idea is that we believe nature is parsimonous in the sense that the generating mechanism of the cause and the mechanism mapping the cause to the effect are independent.  Based on this independence assumption and some other extra working assumption, it is possible to derive the so called cause-effect asymmetry.

Although positive results were reported on synthetic dataset and some real data in \cite{mitrovic2018causal}, there are some potential problems regarding the discriminative power of the RKHS norm-based method which is essentially calculating the variance of the conditional mean embedding norms. Back to the motivating example, we notice that in the anti-causal direction, the conditional density in red and the one in green are symmetric with respect to the y-axis and the structural variation between the red and the green one is significant. However, the norms of the RKHS mean embeddings of these two conditional distributions would be equal which leads to some issues of the discriminative power of the direct norm-based method \cite{mitrovic2018causal},i.e. a direct norm-based approach might lose the discriminative power to distinguish two distributions with significant structural variability. We give a formal justification of this conjecture in the next section.

\subsection{Discriminative Power Issues of the RKHS-norm-variance approach}

The major limitation of the direct norm-based approach is that the mapping of a probability distribution to the norm of its RHKS mean embedding is not injective, i.e., there might two distinct probability distributions sharing the same RKHS mean embeddingnorm  . Consequently, a deviance measure that simply calculate the variance of the RKHS such norms might not be discriminative enough for causal discovery. In the following lemma, we show that the norm of the kernel mean embeddings $\Vert \mu_{p}\Vert_{\mathcal{H}_{\mathcal{X}}}$ and  $\Vert \mu_{q}\Vert_{\mathcal{H}_{\mathcal{X}}}$ which correspond to the probability density function $p(\mathbf{x})$ and $q(\mathbf{x})=p(-\mathbf{x})$ are equal if a stationary (translation invariant) kernel is used.

\begin{lemma}\label{lemma:equal_norm}
Denote the domain of $\mathbf{x}$ as  $\mathcal{X}$, and if $\mathcal{X}$ is symmetric with respect to the origin, given two probability densities $p(\mathbf{x})$ and $q(\mathbf{x})$ where $q(\mathbf{x})=p(-\mathbf{x})$, we attain
\begin{equation}
\Vert \mu_{p} \Vert_{\mathcal{H}_{\mathcal{X}}} = \Vert \mu_{q} \Vert_{\mathcal{H}_{\mathcal{X}}},
\end{equation}
where $\mu$ is the kernel mean embedding with respect to a stationary kernel $k(\mathbf{x},\mathbf{x}')=k(\mathbf{x}-\mathbf{x}')$.
\end{lemma}
\begin{proof}
According to Bochner's theorem \cite{rudin1962fourier}, for a stationary kernel $k(\mathbf{x}-\mathbf{x}')$, we have:
\[
\phi(\mathbf{x}) = [cos(\bm{\omega}_1^T\mathbf{x}), sin(\bm{\omega}_1^T\mathbf{x}),\cdots,cos(\bm{\omega}_{N_{\mathcal{H}}}^T\mathbf{x}), sin(\bm{\omega}_{N_{\mathcal{H}}}^T\mathbf{x})]^T,
\]
where $N_{\mathcal{H}}$ is the dimension of the feature space. We attain $\mu_p = [\rho_1,\varsigma_1,\rho_2,\varsigma_2\cdots,\rho_{N_{\mathcal{H}}},\varsigma_{N_{\mathcal{H}}}]^T$, where $\rho_i = \int cos(\bm{\omega}^T_i \mathbf{x}) p(\mathbf{x})d\mathbf{x}$ and $\varsigma_i = \int sin(\bm{\omega}^T_i \mathbf{x})p(\mathbf{x})d\mathbf{x}$ and thus $\Vert \mu_p \Vert_{\mathcal{H}}=\sum_{i=1}^{N_{\mathcal{H}}} (\rho_i^2 + \varsigma_i^2)$.

Similarly, we have $\mu_q = [\rho_1, -\varsigma_1, \cdots, \rho_{N_{\mathcal{H}}}, -\varsigma_{N_{\mathcal{H}}}]^T$ and thus $\Vert \mu_q \Vert_{\mathcal{H}}=\sum_{i=1}^{N_{\mathcal{H}}} (\rho_i^2 + \varsigma_i^2)$.
Consequently, we show that $\Vert \mu_p \Vert_{\mathcal{H}_{\mathcal{X}}}=\Vert \mu_q \Vert_{\mathcal{H}_{\mathcal{X}}}$.
\end{proof}

According to Lemma ~\ref{lemma:equal_norm}, we see that even thought two probability densities are very different, e.g. for skewed distribution, $p(\mathbf{x})$ and $q(\mathbf{x})$ are different, but they share the same norm.

Similar conclusion can be drawn for more general cases and is justified in Lemma~\ref{lemma:non_injective}.
\begin{wrapfigure}{r}{0.5\textwidth}
%\begin{figure}[h]
  \centering
  \includegraphics[width=0.5\textwidth]{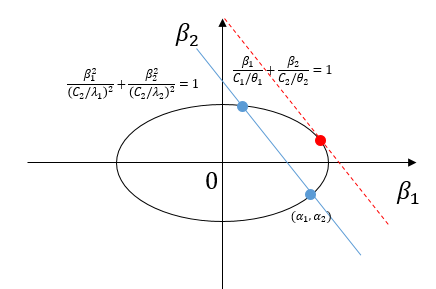}
  \caption{A geometric interpretation}\label{Fig:ellipse}
\vspace{-6mm}
%\end{figure}
\end{wrapfigure}

\begin{lemma}\label{lemma:non_injective}
Given an arbitrary probability density $p(\mathbf{x})\in \mathcal{H}_{\mathcal{X}}$, where $\mathcal{H}_{\mathcal{X}}$ is a Reproducing Kernel Hilbert Space (RKHS) entailed by a positive definite kernel $k(\mathbf{x},\mathbf{x}')$, then with high probability there exists at least one probability density $q(\mathbf{x})\in \mathcal{H}_{\mathcal{X}}$ and $q(\mathbf{x})\neq p(\mathbf{x})$ such that
\[
\Vert \mu_{p} \Vert_{\mathcal{H}_{\mathcal{X}}} = \Vert \mu_q \Vert_{\mathcal{H}_{\mathcal{X}}},
\]
where $\mu_p$ and $\mu_q$ are the kernel mean embeddings of $p(\mathbf{x})$ and $q(\mathbf{x})$ respectively.
\end{lemma}

\begin{proof}
Given a positive definite kernel $k(\mathbf{x},\mathbf{x}')$, according to Mercer's Theorem \cite{scholkopf2001learning}, we have $\mathbf{x}\mapsto \phi(\mathbf{x})$, where $\phi(\mathbf{x})=[\sqrt{\lambda_1}\phi_1(\mathbf{x}),\sqrt{\lambda}_2\phi_2(\mathbf{x}),\cdots,\sqrt{\lambda_{N_{\mathcal{H}}}}\phi_{N_{\mathcal{H}}}(\mathbf{x})]$
,where $\lambda_i>0$ and $\int \phi_i(\mathbf{x}) \phi_j(\mathbf{x})d\mathbf{x}=\delta_{ij}$. We further assume $\phi_i(\mathbf{x})$ are integrable, $\forall i$, then we write $\theta_i = \int \sqrt{\lambda_i}\phi_i(\mathbf{x})d\mathbf{x}$ and thus $\bm{\theta}=\int \phi(\mathbf{x})d\mathbf{x}$. For an arbitrary probability density $p(\mathbf{x})\in \mathcal{H}_{\mathcal{X}}$, we can represent it as:
\[
p(\mathbf{x}) = {1}/{\bm{\alpha}^T\bm{\theta}} \phi(\mathbf{x})^T\bm{\alpha},
\]
where $\bm{\alpha}$ is a vector of coefficient. By definition, the RKHS mean embedding of $p(\mathbf{x})$ is obtained as $\mu_{p} = \int \phi(\mathbf{x})p(\mathbf{x})d\mathbf{x} = \frac{1}{\bm{\alpha}^T\bm{\theta}}\bm{\Lambda}\bm{\alpha}$, where $\bm{\Lambda}$ is a diagonal matrix with $[\bm{\Lambda}]_{ii}=\lambda_i$. The norm of $\mu_{p}$ can be easily calculated as $\Vert \mu_{p} \Vert_{\mathcal{H}_{\mathcal{X}}}^2 = 1/(\bm{\alpha}^T\bm{\theta})^2 \bm{\alpha}^T\bm{\Lambda}^2\bm{\alpha}=1/(\bm{\alpha}^T\bm{\theta})^2 \sum_{i=1}^{N_{\mathcal{H}}} \lambda_i^2 \alpha_i^2$.

Now we construct another probability density $q(\mathbf{x}) = \frac{1}{\bm{\beta}^T\bm{\theta}}\phi(\mathbf{x})^T\bm{\beta}$. Without loss of generality, we assume that $\beta_i = \alpha_i,~\forall i\geq 2$. Similarly, we have $\Vert \mu_{q} \Vert_{\mathcal{H}}^2 = 1/(\bm{\beta}^T\bm{\theta})^2 \bm{\beta}^T\bm{\Lambda}^2\bm{\beta}=1/(\bm{\beta}^T\bm{\theta})^2 \sum_{i=1}^{N_{\mathcal{H}}} \lambda_i^2 \beta_i^2$. In order to make $\Vert \mu_{p} \Vert_{\mathcal{H}} = \Vert \mu_{q} \Vert_{\mathcal{H}}$, we construct $\beta_1$ and $\beta_2$ in the way that:
\begin{equation}\label{Eq:beta_equations}
\begin{split}
& \theta_1 \beta_1 + \theta_2 \beta_2 = \theta_1\alpha_1 + \theta_2\alpha_2 = C_1\\
&\lambda_1^2 \beta_1^2 + \lambda_2^2 \beta_2^2 = \lambda_1^2 \alpha_1^2 + \lambda_2^2 \alpha_2^2 = C_2^2,\\
\end{split}
\end{equation}
where $C_2>0$.
We attain $({\lambda_1}/{C_2})^2 \beta_1^2 + ({\lambda_2}/{C_2})^2 \beta_2^2 = 1$. Let $\beta_1 = C_2/\lambda_1 sin(\varphi)$ and $\beta_2 = C_2/\lambda_2 cos(\varphi)$, we obtain:
\begin{equation}\label{Eq:triagonal_equation}
{\theta_1C_2}/{\lambda_1} sin(\varphi) + {\theta_2C_2}/{\lambda_2}cos(\varphi) = C_1.
\end{equation}

In order to ensure a solution exists for Eq. \ref{Eq:triagonal_equation}, we need to prove that $\left\lvert {C_1}/{\sqrt{(\frac{\theta_1C_2}{\lambda_1})^2 + (\frac{\theta_2C_2}{\lambda_2})^2}} \right\rvert \leq 1$.
We show that
\[
\begin{split}
&{\theta_1^2C_2^2}/{\lambda_1^2} + {\theta_2^2C_2^2}/{\lambda_2^2} - C_1^2 = {\theta_1^2}/{\lambda_1^2}(\lambda_1^2\alpha_1^2 + \lambda_2^2\alpha_2^2) + {\theta_2^2}/{\lambda_2^2}(\lambda_1^2\alpha_1^2 + \lambda_2^2\alpha_2^2) - (\theta_1\alpha_1+\theta_2\alpha_2)^2\\
&={\lambda_1^2}/{\lambda_2^2}\theta_2^2\alpha_1^2 + {\lambda_2^2}/{\lambda_1^2}\theta_1^2\alpha_2^2 - 2\theta_1\theta_2\alpha_1\alpha_2 \geq 2\vert \theta_1\theta_2\alpha_1\alpha_2 \vert - 2\theta_1\theta_2\alpha_1\alpha_2 \geq 0.
\end{split}
\]
Consequently, we prove that Eq.\ref{Eq:triagonal_equation} holds and there exists two solutions, i.e. $\varphi = \mathop{arcsin}\left( {C_1}/{\sqrt{(\frac{\theta_1C_2}{\lambda_1})^2 + (\frac{\theta_2C_2}{\lambda_2})^2}}\right) - \omega $ and $\varphi = \pi  - \mathop{arcsin} \left({C_1}/{\sqrt{(\frac{\theta_1C_2}{\lambda_1})^2 + (\frac{\theta_2C_2}{\lambda_2})^2}} \right)- \omega$, where $\mathop{sin}(\omega)=\theta_2/(\lambda_2\sqrt{\theta_1^2/\lambda_1^2 + \theta_2^2/\lambda_2^2})$ and $\mathop{cos}(\omega)=\theta_1/(\lambda_1\sqrt{\theta_1^2/\lambda_1^2 + \theta_2^2/\lambda_2^2})$. Two solutions collapse to one if and only if $\lvert {C_1}/{\sqrt{(\frac{\theta_1C_2}{\lambda_1})^2 + (\frac{\theta_2C_2}{\lambda_2})^2}} \rvert=1$ which rarely happens as it requires mutual adjustment of the probability density function $p(\mathbf{x})$ and the kernel function.
\end{proof}

The intuitive interpretation of the proof can also be elucidated in Fig.~\ref{Fig:ellipse}. The solution $(\beta_1,\beta_2)$ of the first equation $\theta_1 \beta_1 + \theta_2 \beta_2=C_1$ forms a line and the solution of the second equation $\lambda_1^2 \beta_1^2 + \lambda_2^2 \beta_2^2=C_2^2$ forms an ellipse and thus the solution of Eq. \ref{Eq:beta_equations} is the intersection of the line and the ellipse. Note that the intersection should happen as $(\alpha_1,\alpha_2)$ is already a solution to Eq. \ref{Eq:beta_equations}. With high probability, there are two distinct intersection points as shown in Fig. \ref{Fig:ellipse} except for some rare cases that the points collapse to a single point when the line is the tangent line of the ellipse. This is rare because it requires mutual adjustment between $\bm{\alpha}$, $\bm{\theta}$  and $\bm{\lambda}$ which in turn essentially requires the mutual adjustment between $\phi(\mathbf{x})$ and $p(\mathbf{x})$. According to Lemma ~\ref{lemma:non_injective}, we see that the RKHS norm which is directly applied to the conditional distribution instantiated at different value is not discriminative enough. There are conditional distributions with significant distinction but they can have equal norms and thus it leads to some problems for the proposed KCDC algorithm in \cite{mitrovic2018causal}.

\subsection{Causal Discovery via  Kernel Intrinsic Invariance Measure}

Realizing the limitation of the norm based approach, we propose our method which measures the norm of the difference of the kernel mean embeddings corresponding to conditional distributions instantiated at different values, instead of measuring the difference of their norms. However, a naive application of this idea might not work because even in the causal direction, conditional distributions instantiated at different cause values are not NOT identical. They could be different with each other in terms of the location and the scale, although we are more interested in higher order moments that are more relevant to the shape of the density function. For example, in a toy example proposed in \cite{mitrovic2018causal}, we have $y = x^3 + x + \epsilon$, where $\epsilon \sim \mathcal{N}(0,1)$. Even in the causal direction, we get $p(y|x)=\exp(-0.5(y - x^3 -x)^2)$. Conditional distributions instantiated at different $x$ are all Gaussian distributions but they have different means and thus they are not identical, neither are their kernel mean embeddings and the corresponding norms (this can be easily verified if one use a polynomial kernel).
However, the location and scale information of a distribution are not that interesting to us when it comes to causal discovery as we are more keen on the higher order statistics that reflect shape information.

This observation motivates our method to capture more intrinsic information of the probability density function. Mathematically, we define the following score that capture the ``intrinsic" variation of the conditional distribution instantiated at different values of the conditioning variable $\mathbf{x}$ or $\mathbf{y}$ for two hypothetic directions. Without loss of generality, we show definition of the score in the direction of $\mathbf{x}\rightarrow \mathbf{y}$:

\begin{equation}\label{Eq:our_score}
\mathcal{S}_{\mathbf{x}\rightarrow \mathbf{y} } := \min_{\mathbf{W}}\sum_{i=1}^{n} \Vert \mathbf{W}^T\mu_{Y|\mathbf{x}_i} - \frac{1}{n}\sum_{j=1}^n \mathbf{W}^T\mu_{Y|\mathbf{x}_j}\Vert_{\mathcal{H}_{\mathcal{Y}}}^2.
\end{equation}
The interpretation of Eq.\ref{Eq:our_score} is that we calculate the norm of the difference of conditional distributions at different values. The score is zero if and only if all conditional distributions are the same according the injectiveness of the kernel mean embedding. The matrix $\mathbf{W}$ is introduced to find the subspace which removes the effect of some trivial deviance like location and scale. Empirically, we can calculate the kernel embedding of the conditional distribution instantiated at different $\mathbf{x}$ as $\mu_{Y|\mathbf{x}_i} = \mathbf{\Psi}(\mathbf{K}_{x}+\lambda \mathbf{I})^{-1} \mathbf{k}_{\mathbf{x}_i}$,
where $\bm{\Psi}=[\psi(\mathbf{y}_1),\cdots, \psi(\mathbf{y}_n)]$, $\mathbf{K}_{\mathbf{x}}$ is the kernel gram matrix of $\mathbf{x}$ and $\mathbf{k}_{\mathbf{x}_i}^T=k(\mathbf{X},\mathbf{x}_i)$. Note that the solution of the above optimization problem lies in the span of $\bm{\Psi}$, and thus we can represent $\mathbf{W}$ by a linear combination of $\psi(\mathbf{y}_i)$, i.e. $\mathbf{W} = \bm{\Psi}\tilde{\mathbf{W}}$, where $\tilde{\mathbf{W}}\in R^{n\times p }$ is the coefficient matrix.  Consequently, we attain
\begin{equation}
\begin{split}
\mathcal{S}_{\mathbf{x}\rightarrow \mathbf{y} } &:= \min_{\tilde{\mathbf{W}}}\frac{1}{n}\sum_{i=1}^{n} \Vert \tilde{\mathbf{W}}^T\mathbf{K}_y(\mathbf{K}_x + \lambda \mathbf{I})^{-1}\mathbf{k}_{\mathbf{x}_i} - \frac{1}{n}\sum_{j=1}^n \tilde{\mathbf{W}}^T\mathbf{K}_y(\mathbf{K}_x + \lambda \mathbf{I})^{-1}\mathbf{k}_{\mathbf{x}_j}\Vert_{\mathcal{H}_{\mathcal{Y}}}^2\\
%& = \min_{\mathbf{W}} \frac{1}{n}\sum_{i=1}^n \Vert \tilde{\mathbf{W}}^T\mathbf{K}_{y}(\mathbf{K}_x + \lambda \mathbf{I})^{-1} (\mathbf{k}_{\mathbf{x}_i} - \frac{1}{n} \sum_{j=1}^n \mathbf{k}_{\mathbf{x}_j}) \Vert_{\mathcal{H}}^2\\
%& =  \min_{\mathbf{W}} \sum_{i=1}^n \frac{1}{n}\mathop{Tr} \left( \tilde{\mathbf{W}}^T\mathbf{K}_{y}(\mathbf{K}_x + \lambda \mathbf{I})^{-1} (\mathbf{k}_{\mathbf{x}_i} - \frac{1}{n} \sum_{j=1}^n \mathbf{k}_{\mathbf{x}_j})(\mathbf{k}_{\mathbf{x}_i} - \frac{1}{n} \sum_{j=1}^n \mathbf{k}_{\mathbf{x}_j})^T(\mathbf{K}_x + \lambda \mathbf{I})^{-1}  \mathbf{K}_{y} \tilde{\mathbf{W}}\right) \\
%& = \min_{\mathbf{W}}  \frac{1}{n}\mathop{Tr} \left( \tilde{\mathbf{W}}^T\mathbf{K}_{y}(\mathbf{K}_x + \lambda \mathbf{I})^{-1} ( \sum_{i=1}^n(\mathbf{k}_{\mathbf{x}_i} - \frac{1}{n} \sum_{j=1}^n \mathbf{k}_{\mathbf{x}_j})(\mathbf{k}_{\mathbf{x}_i} - \frac{1}{n} \sum_{j=1}^n \mathbf{k}_{\mathbf{x}_j})^T)(\mathbf{K}_x + \lambda \mathbf{I})^{-1}  \mathbf{K}_{y} \tilde{\mathbf{W}}\right) \\
&= \min_{\tilde{\mathbf{W}}} \frac{1}{n} \mathop{Tr} \left( \tilde{\mathbf{W}}^T\mathbf{K}_{y}(\mathbf{K}_x + \lambda \mathbf{I})^{-1} \mathbf{K}_x\mathbf{H}\mathbf{K}_x(\mathbf{K}_x + \lambda \mathbf{I})^{-1}  \mathbf{K}_{y} \tilde{\mathbf{W}}\right)
\end{split}
\end{equation}

To avoid trivial solution $\tilde{\mathbf{W}}=\mathbf{0}$, we pose the constraint that $\tilde{\mathbf{W}}^T\tilde{\mathbf{W}}=\mathbf{I}$. Consequently, we have

\begin{equation}
\begin{split}
\mathcal{S}_{\mathbf{x}\rightarrow \mathbf{y}}&= \min_{\tilde{\mathbf{W}}} \frac{1}{n} \mathop{Tr} \left( \tilde{\mathbf{W}}^T\mathbf{K}_{y}(\mathbf{K}_x + \lambda \mathbf{I})^{-1} \mathbf{K}_x\mathbf{H}\mathbf{K}_x(\mathbf{K}_x + \lambda \mathbf{I})^{-1}  \mathbf{K}_{y} \tilde{\mathbf{W}}\right) ~ \mbox{s.t.} \tilde{\mathbf{W}}^T\tilde{\mathbf{W}} = \mathbf{I}.\\
\end{split}
\end{equation}

The intuitive interpretation of the proposed method is that we use the projection matrix $\mathbf{W}$ to find intrinsic deviance of the conditional mean embedding. The intrinsic deviance captures higher order statistics which is more relevant to the shape of the probability distribution function while discards the some less important information, e.g. the location and scale of a distribution.
As an illustrative example, suppose we use a polynomial kernel $k({x},{x}')=(xx'+1)^{d}$, it can be easily shown that we essentially map $x$ to a space with polynomials of the feature, i.e. $x\mapsto [1, x, x^2, \cdots, x^d]$. Therefore, the kernel mean embedding of a distribution is in fact a vector of moments up to degree $d$, i.e.
\[
\mu_X := \int \phi(x)p(x)dx = [1, m_1, m_2, \cdots, m_d]^T,
\]
where $m_i$ denotes the $i$-th order moment. If the conditional distributions only differ from each other with mean and standard deviation (the first and second order moments), then the projection matrix $\mathbf{W}$ is expected to find the subspace that contains only higher order moments.

However, how to decide the rank of $\mathbf{W}$ is an open question. In this paper, we propose a simple but effective algorithm to choose the right rank, see Alg. \ref{Alg:Framwork}. The basic idea is to project of the subspace corresponding to the smallest $k$ eigenvalues which preserve at least $90\%$ of the energy of the whole spectrum.

\begin{algorithm}[htb]
\caption{ Framework of the Kernel Intrinsic Invariance Measure (KIIM)}
\label{Alg:Framwork}
\begin{algorithmic}[1] %这个1 表示每一行都显示数字
\REQUIRE ~~\\ %算法的输入参数：Input
The set of samples $X=[\mathbf{x}_1,\mathbf{x}_2,\cdots,\mathbf{x}_n]$, $Y=[\mathbf{y}_1,\mathbf{y}_2,\cdots,\mathbf{y}_n]$;\\
The regularization hyperparameter $\lambda$;\\
\ENSURE ~~\\ %算法的输出：Output
The inferred causal direction, $\mathcal{S}_{\mathbf{x}\rightarrow \mathbf{y}}$, and $\mathcal{S}_{\mathbf{y}\rightarrow \mathbf{x}}$.
\STATE Compute Kernel Gram matrix $\mathbf{K}_y$, $\mathbf{K}_x$ and the centering matrix $\mathbf{H}$;
%\label{ code:fram:extract }%对此行的标记，方便在文中引用算法的某个步骤
\STATE Compute the Kernel Intrinsic Invariance Matrix $\mathbf{M}^{(1)}=\mathbf{K}_{y}(\mathbf{K}_x + \lambda \mathbf{I})^{-1} \mathbf{K}_x\mathbf{H}\mathbf{K}_x(\mathbf{K}_x + \lambda \mathbf{I})^{-1}  \mathbf{K}_{y}$ for the hypothetic direction $\mathbf{x}\rightarrow \mathbf{y}$;

\STATE Conduct eigen-decomposition of $\mathbf{M}^{(1)}=\mathbf{U}^{(1)}\bm{\Pi}^{(1)}(\mathbf{U}^{(1)})^T$ and calculate $\mathcal{S}_{\mathbf{x}\rightarrow \mathbf{y}}=\sum_{i=k_1}^n \pi^{(1)}_i$, where $\pi^{(1)}_i,~\forall i\geq k_1$ are the bottom-$k$ smallest eigenvalues such that $\sum_{i=k_1}^n \pi^{(1)}_i / \sum_{i=1}^n \pi^{(1)}_i\geq 0.9$ and $k_1$ is the largest number when the inequality holds.

\STATE Repeat Step 3 for the other direction.

%\STATE Conduct eigen-decomposition of $\mathbf{M}^{(2)}=\mathbf{U}^{(2)}\bm{\Pi}^{(2)}(\mathbf{U}^{(2)})^T$ and $\mathcal{S}_{\mathbf{y}\rightarrow \mathbf{x}}=\sum_{i=k_2}^n \pi^{(2)}_i$, where $\pi^{(2)}_i,~\forall i\geq k_2$ are the bottom-$k$ smallest eigenvalues such that $\sum_{i=k_2}^n \pi^{(2)}_i / \sum_{i=1}^n \pi^{(2)}_i\geq 0.9$ and $k_2$ is the largest number when the inequality holds.
\STATE The causal direction is inferred as $\mathbf{x}\rightarrow \mathbf{y}$ if $\mathcal{S}_{\mathbf{x}\rightarrow}\mathbf{y}<\mathcal{S}_{\mathbf{y}\rightarrow}\mathbf{x}$ or $\mathbf{y}\rightarrow \mathbf{x}$ if $\mathcal{S}_{\mathbf{x}\rightarrow}\mathbf{y}>\mathcal{S}_{\mathbf{y}\rightarrow}\mathbf{x}$. No conclusion will be made if $\mathcal{S}_{\mathbf{x}\rightarrow}\mathbf{y}=\mathcal{S}_{\mathbf{y}\rightarrow}\mathbf{x}$.
\RETURN The causal direction, $\mathcal{S}_{\mathbf{x}\rightarrow \mathbf{y}}$, $\mathcal{S}_{\mathbf{y}\rightarrow \mathbf{x}}$ ; %算法的返回值
\end{algorithmic}
\end{algorithm}

\subsection{Robust Kernel Intrinsic Invariance Measure by Importance Reweighting}

Real world data is usually contaminated with noise and outliers. The estimation of the kernel mean embedding might be significantly biased due to the outliers in data. Furthermore, when estimating the conditional mean embedding, we want to eliminate any effect of the marginal distribution of the hypothetic cause due to finite sample size. We adopt an importance reweighting scheme as follows:
\begin{equation}
%\begin{split}
\mathcal{C}_{YX}^{ref} := \int \phi(\mathbf{y})\otimes \phi(\mathbf{x}) p(\mathbf{y}|\mathbf{x})u(\mathbf{x})d\mathbf{x}, ~~\mathcal{C}_{XX}^{ref} := \int \phi(\mathbf{x})\otimes \phi(\mathbf{x})u(\mathbf{x})d\mathbf{x},\\
%\end{split}
\end{equation}
and $\mu_{Y|\mathbf{x}}^{ref} = \mathcal{C}_{YX}^{ref}\left(\mathcal{C}_{XX}^{ref}\right)^{-1}\phi(\mathbf{x})$, where $u(\mathbf{x})$ is a reference distribution. The empirical estimation is then obtained by
\begin{equation}
\hat{\mu}_{Y|\mathbf{x}}^{ref} = \bm{\Psi}\mathbf{H}\mathbf{R}^{1/2}(\mathbf{H}\mathbf{R}^{1/2}\mathbf{K}_x\mathbf{R}^{1/2}\mathbf{H} + \lambda n \mathbf{I})^{-1}\mathbf{R}^{1/2}\mathbf{H}\mathbf{k}_{\mathbf{x}},
\end{equation}
where $\mathbf{R}$ is a diagonal reweighting matrix with $[\mathbf{R}]_{ii} = u(\mathbf{x}_i)/p(\mathbf{x}_i)$. The main body of the algorithm does not change except for the calculation of $\mathbf{M}^{(1)}$ and $\mathbf{M}^{(2)}$ in Alg. \ref{Alg:Framwork}. We name this variant of our algorithm Rw-KIIM meaning Reweighted Kernel Intrinsic Invariance Measure.

\section{Experiment}\label{Sec:experiment}
In this section, we conduct experiments using both synthetic data and a real world dataset called Tuebigen Cause Effect Pairs (TCEP) . We compare our methods with some state-of-the-art methods including KCDC\footnote{Although positive results were reported in \cite{mitrovic2018causal}, unfortunately we are not able to reproduce the results reported in the paper.}, IGCI \cite{janzing2012information}, ANM \cite{hoyer2009nonlinear} and LiNGAM\cite{shimizu2006linear} \footnote{https://github.com/Diviyan-Kalainathan/CausalDiscoveryToolbox}. For IGCI, we use the entropy based methods with two different reference distribution (Gaussian and Uniform distribution). We use $1e-3$ for the regularization hyperparameter when calculating the conditional mean embedding in Eq.\ref{Eq:empirical_estimation}. In the following experiment, we use the composite kernel for KIIM which is the multiplication of the RBF kernel $k(\mathbf{x},\mathbf{x}')=\exp(-\Vert \mathbf{x} - \mathbf{x}'\Vert^2/\sigma^2)$ with median heuristic for kernel width and a log kernel $k(\mathbf{x},\mathbf{x}')=-\log (\Vert \mathbf{x} - \mathbf{x}'\Vert^2+1)$ and a rational quadratic kernel   $k(\mathbf{x},\mathbf{x}')= 1 - \frac{\Vert \mathbf{x} - \mathbf{x}'\Vert^2}{\Vert \mathbf{x} - \mathbf{x}'\Vert^2+1}$. For KCDC, we use log kernel for the input and rational quadratic kernel for the output as in \cite{mitrovic2018causal}.
\begin{table}[h]
\vspace{-3mm}
\caption{Performance of synthetic dataset}
\resizebox{\textwidth}{!}{%
\begin{tabular}{ccccccc}
\hline
\hline
ANM-1& KCDC & KIIM & Rw-KIIM &IGCI(entropy,Gaussian) &IGCI(entropy, Uniform) &ANM\\
\hline
Gaussian & $93.5\%\pm 2.5\%$ &$100.0\%\pm 0.0\%$ &$100.0\%\pm 0.0\%$ &$98.1\%\pm 1.4\%$ &$100.0\%\pm 0.0\%$ & $100.0\%\pm 0.0\%$\\
Uniform & $61.1\%\pm 3.9\%$ &$100.0\%\pm 0.0\%$ &$100.0\%\pm 0.0\%$ &$99.4\%\pm 0.97\%$ &$100.0\%\pm 0.0\%$ & $100.0\%\pm 0.0\%$\\
\hline
ANM-2& KCDC & KIIM & Rw-KIIM &IGCI(entropy,Gaussian) &IGCI(entropy, Uniform) &ANM\\
\hline
squared-Gaussian & $59.8\%\pm 4.3\%$ &$89.6\%\pm 2.4\%$ &$89.8\%\pm 2.4\%$ &$74.8\%\pm 4.6\%$ &$95.3\%\pm 2.1\%$ & $70.1\%\pm 5.3\%$\\
Uniform & $57.3\%\pm 3.1\%$ &$56.8\%\pm 4.5\%$ &$57.0\%\pm 4.2\%$ &$49.4\%\pm 5.6\%$ &$49.4\%\pm 6.1\%$ & $67.6\%\pm 3.2\%$\\
\hline
\hline
MNM-1& KCDC & KIIM & Rw-KIIM &IGCI(entropy,Gaussian) &IGCI(entropy, Uniform) &ANM\\
\hline
Gaussian & $23.5\%\pm 3.1\%$ &$100.0\%\pm 0.0\%$ &$100.0\%\pm 0.0\%$ &$98.1\%\pm 1.4\%$ &$100.0\%\pm 0.0\%$ & $0.2\%\pm 0.4\%$\\
Uniform & $24.6\%\pm 4.1\%$ &$100.0\%\pm 0.0\%$ &$100.0\%\pm 0.0\%$ &$99.9\%\pm 0.3\%$ &$100.0\%\pm 0.0\%$ & $0.0\%\pm 0.0\%$\\
\hline
MNM-2& KCDC & KIIM & Rw-KIIM &IGCI(entropy,Gaussian) &IGCI(entropy, Uniform) &ANM\\
\hline
Gaussian & $60.2\%\pm 6.6\%$ &$100.0\%\pm 0.0\%$ &$100.0\%\pm 0.0\%$ &$100.0\%\pm 0.0\%$ &$100.0\%\pm 0.0\%$ & $1.0\%\pm 0.7\%$\\
Uniform & $97.9\%\pm 1.1\%$ &$100.0\%\pm 0.0\%$ &$100.0\%\pm 0.0\%$ &$100.0\%\pm 0.0\%$ &$100.0\%\pm 0.0\%$ & $28.4\%\pm 5.7\%$\\
\hline
\hline
Complex& KCDC & KIIM & Rw-KIIM &IGCI(entropy,Gaussian) &IGCI(entropy, Uniform) &ANM\\
\hline
Gaussian & $27.6\%\pm 5.8\%$ &$99.8\%\pm 0.4\%$ &$99.8\%\pm 0.4\%$ &$100.0\%\pm 0.0\%$ &$100.0\%\pm 0.0\%$ & $6.9\%\pm 1.9\%$\\
Uniform & $5.1\%\pm 1.4\%$ &$91.4\%\pm 2.0\%$ &$91.7\%\pm 2.0\%$ &$100.0\%\pm 0.0\%$ &$100.0\%\pm 0.0\%$ & $15.1\%\pm 3.8\%$\\
\hline
\end{tabular}
}
\label{Tab:synthetic}
\end{table}

\subsection{Synthetic Data}
In this section, we evaluate pairwise causal discover algorithms on data generated by 5 different data generation mechanisms following \cite{mitrovic2018causal}. Details of these mechanisms are given as follows. ANM-1: $y=x^3 + x + \epsilon$; ANM-2: $y=x + \epsilon$. MNM-1: $y = (x^3 + x)\exp(\epsilon)$; MNM-2: $y=(\sin(10x)+\exp(3x))\exp(\epsilon)$ and CNM: $y=(\log(x+10)+x^2)^{\epsilon}$ and the noise distribution is specified in Tab.\ref{Tab:synthetic}. We generate 100 samples from each data generation mechanism for different algorithm to infer the causal direction. Experiments are conducted for 100 independent trials and results of different algorithms are reported in Tab.\ref{Tab:synthetic}. We observe that IGCI is quite robust and $ANM$ performs well when the data generation mechanism is indeed additive noise model. Unfortunately, we are not able to reproduce positive results reported in \cite{mitrovic2018causal}. In this paper, we are using a direct version of KCDC without majority vote and the confident measure \cite{mitrovic2018causal} because these extra processes are not used in our algorithm IKKM. Even without this extra process, IKKM and Rw-KIIM perform quite well except for the linear ANM with uniform noise.

In order to justify the necessity of using a projection matrix $\mathbf{W}$ to a lower dimensional space, we compare the performance of IKKM with different ranks of $\mathbf{W}$. This results in using the algorithm exploiting eigenvalues ranging from the whole spectrum to only the smallest one. Two mechanisms are used in this experiments as shown in Fig. \ref{Fig:robustness_of_rank}. Interestingly, we find that the algorithm using the whole spectrum does not perform the best but the one discarding the top-1 eigenvalue performs consistently the best. This result justifies our motivation and argument: we need intrinsic deviance/invariance measurement that captures only higher order statistics of the shape of the conditional distribution. Trivial difference arising from the location and the scale might not be beneficial or even harmful for causal discovery.

%\begin{wrapfigure}{r}{0.4\textwidth}
%%\vspace{-3mm}
%  \centering
%  \includegraphics[width=0.4\textwidth]{robustness_of_rank}
%  \caption{Why intrinsic deviation is needed?}\label{Fig:robustness_of_rank}
%%\vspace{-3mm}
%\end{wrapfigure}

\subsection{Tuebinen Cause-Effect Pairs (TCEP)}

In this section, we verify the performance of our algorithm on real world data. We use the open benchmark called Tuebigen Cause-Effect Pairs\footnote{https://webdav.tuebingen.mpg.de/cause-effect/} which has been widely used to evaluate causal discovery algorithms. The whole dataset contains 108 cause-effect pairs  taken from 37 different data sets from various domains \cite{mooij2016distinguishing} with known ground truth. We do not use some pairs as both $x$ and $y$ are high-dimensional variables in pair 52,53,54,55,71,105 and there are missing values in pairs 81, 82 and 83. The ground truth direction in pair 86 is not mentioned in the data description document and thus is not used in our experiments. Fig.\ref{Fig:tcep_accuracy} shows our proposed algorithm outperform the state-of-the-art methods significantly.

%\begin{wrapfigure}{r}{0.5\textwidth}
%%\begin{figure}[h]\label{Fig:TCEP}
%  \centering
%  \includegraphics[width=0.5\textwidth]{TCEP}
%  \caption{Accuracy of the TCEP dataset}\label{Fig:tcep_accuracy}
%\end{wrapfigure}

\begin{figure}
\centering
\begin{subfigure}{.5\textwidth}
  \centering
  \includegraphics[width=\linewidth]{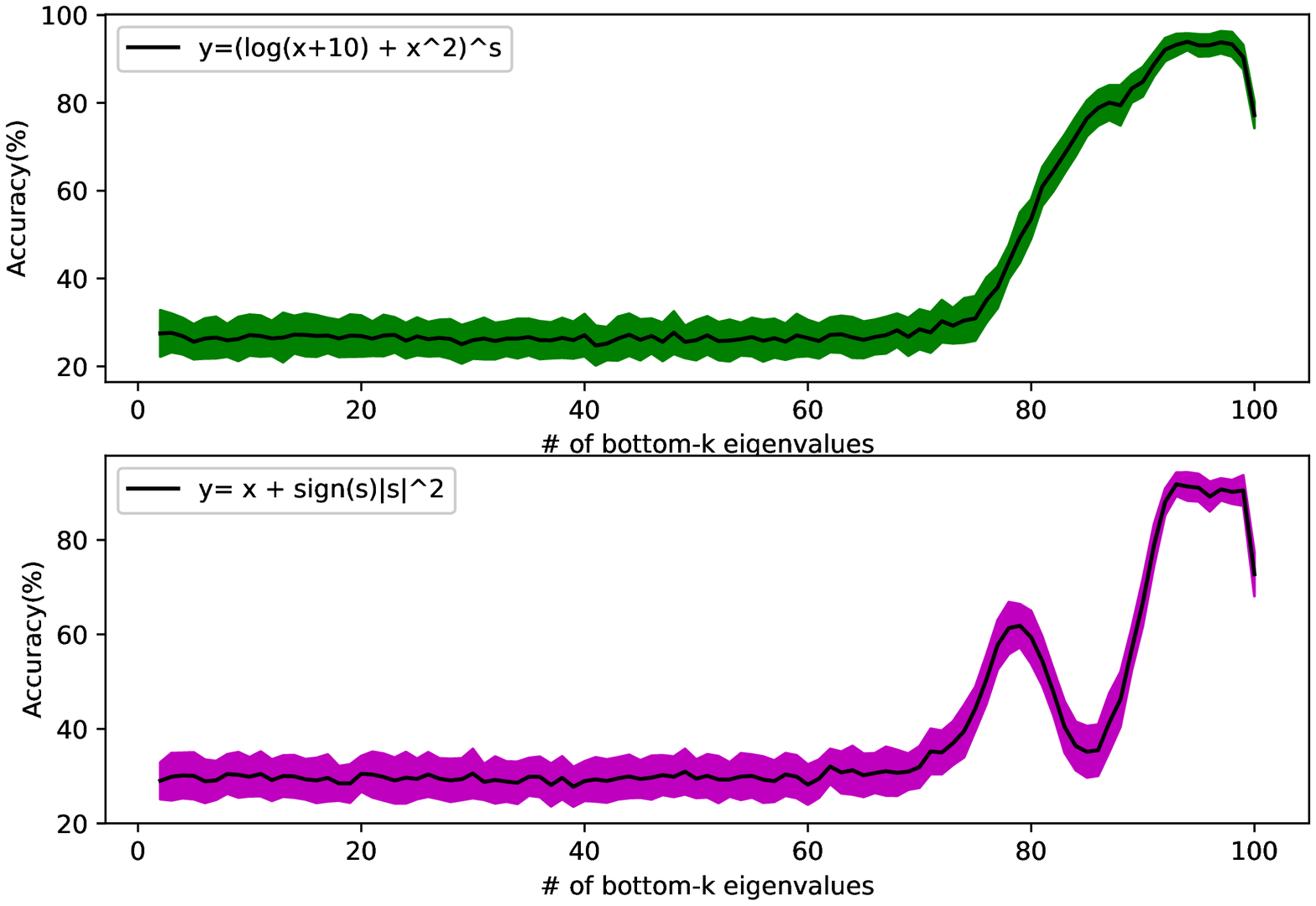}
\caption{Why intrinsic deviance is needed?}\label{Fig:robustness_of_rank}
\end{subfigure}%
\begin{subfigure}{.5\textwidth}
  \centering
  \includegraphics[width=\linewidth]{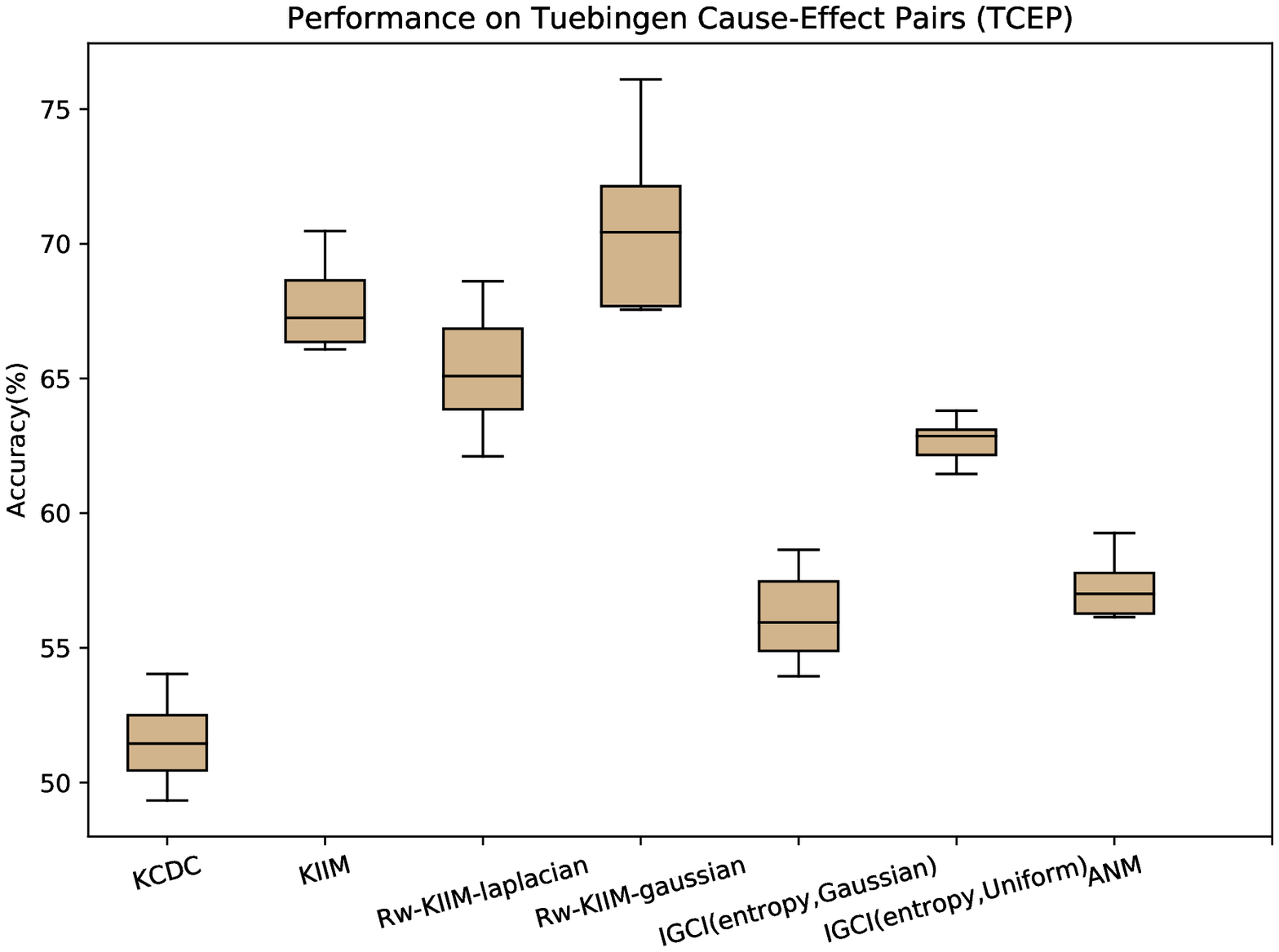}
\caption{Accuracy of the TCEP dataset}\label{Fig:tcep_accuracy}
\end{subfigure}

\end{figure}

\section{Conclusion}\label{Sec:conclusion}
In this paper, we focus on causal discovery for cause-effect pairs along the direction of Independent Mechanism (IM) principle. We
prove that the existing norm-based state-of-the-art which only compare the norms of conditional mean embedding might lose discriminative power. To solve this problem, we propose a Kernel Intrinsic Invariance Measure (KIIM) to capture the intrinsic invariance of the conditional distribution, i.e. the higher order statistics corresponding to the shapes of the density functions. Experiments with synthetic data and real data justify the effectiveness of our proposed algorithm and supports our argument that we indeed need to look for higher order statistics/intrinsic invariance for causal discovery.

\let\oldthebibliography=\thebibliography
\let\endoldthebibliography=\endthebibliography
\renewenvironment{thebibliography}[1]{%
   \begin{oldthebibliography}{#1}%
     \setlength{\itemsep}{0.3ex}%
}%
{%
   \end{oldthebibliography}%
}

\bibliographystyle{abbrv}
\bibliography{gibbsOpt}

\end{document}